\documentclass[letterpaper, 10pt, conference]{ieeeconf}

\IEEEoverridecommandlockouts
\usepackage{times}
\usepackage{amsmath,amssymb,amsfonts}
\usepackage{algorithmic}
\usepackage{algorithm}
\usepackage{graphicx}
\usepackage{textcomp}
\usepackage{xcolor}
\usepackage{booktabs}
\usepackage{placeins}
\usepackage{multirow}
\usepackage{cite}
\usepackage[colorlinks=true,linkcolor=red,citecolor=green!50!black,urlcolor=red]{hyperref}
\usepackage{bm}

\newtheorem{proposition}{Proposition}
\newcommand{\qed}{\hfill\QED}

\newcommand{\R}{\mathbb{R}}
\newcommand{\E}{\mathbb{E}}

\newcommand{\Dset}{\mathcal{D}}
\newcommand{\xgoal}{x_{\mathrm{goal}}}
\newcommand{\xzero}{x_{0}}
\newcommand{\Yhat}{\hat{Y}}
\newcommand{\Xhat}{\hat{X}}
\newcommand{\wh}{\hat{w}}

\title{\LARGE \bf Behavioral Score Diffusion: Model-Free Trajectory Planning via Kernel-Based Score Estimation from Data}

\author{Shihao Li, Jiachen Li, Jiamin Xu, Dongmei Chen% <-this % stops a space
\thanks{All authors are with The University of Texas at Austin.
{\tt\small \{shihaoli01301, jiachenli\}@utexas.edu, jiaminxu@my.utexas.edu, dmchen@me.utexas.edu}}%
}

\begin{document}

\maketitle

% ============================================================
% ABSTRACT
% ============================================================
\begin{abstract}
Diffusion-based trajectory optimization has emerged as a powerful planning paradigm, but existing methods require either learned score networks trained on large datasets or analytical dynamics models for score computation. We introduce \emph{Behavioral Score Diffusion} (BSD), a training-free and model-free trajectory planner that computes the diffusion score function directly from a library of trajectory data via kernel-weighted estimation. At each denoising step, BSD retrieves relevant trajectories using a triple-kernel weighting scheme---diffusion proximity, state context, and goal relevance---and computes a Nadaraya-Watson estimate of the denoised trajectory. The diffusion noise schedule naturally controls kernel bandwidths, creating a multi-scale nonparametric regression: broad averaging of global behavioral patterns at high noise, fine-grained local interpolation at low noise. This coarse-to-fine structure handles nonlinear dynamics without linearization or parametric assumptions. Safety is preserved by applying shielded rollout on kernel-estimated state trajectories, identical to existing model-based approaches. We evaluate BSD on four robotic systems of increasing complexity (3D--6D state spaces) in a parking scenario. BSD with fixed bandwidth achieves 98.5\% of the model-based baseline's average reward across systems while requiring no dynamics model, using only 1{,}000 pre-collected trajectories. BSD substantially outperforms nearest-neighbor retrieval (18--63\% improvement), confirming that the diffusion denoising mechanism is essential for effective data-driven planning.
\href{https://sheehow.github.io/behavioral-score-diffusion/}{\textbf{[Project Page]}}
\href{https://sheehow.github.io/behavioral-score-diffusion/code/}{\textbf{[Code]}}
\end{abstract}

% ============================================================
% I. INTRODUCTION
% ============================================================
\section{Introduction}

Trajectory optimization is fundamental for autonomous robots in constrained environments, but classical approaches require explicit dynamics models~\cite{betts1998survey}. Obtaining accurate models is expensive or infeasible for many real-world systems---articulated vehicles with complex tire-ground interactions, soft robots, or systems with proprietary dynamics.

Diffusion-based trajectory optimization~\cite{janner2022planning, pan2024mbd} reformulates planning as iterative denoising. Model-Based Diffusion (MBD)~\cite{pan2024mbd} computes the score function via reward-weighted importance sampling over dynamics rollouts, achieving strong performance without neural network training. Safe-MPD~\cite{kim2026safempd} extends MBD with a safety shield enforcing collision-free trajectories during denoising. However, both require an analytical dynamics model---limiting applicability and coupling planning quality to model fidelity.

We propose \emph{Behavioral Score Diffusion} (BSD), which eliminates this model dependency entirely (Fig.~\ref{fig:overview}). BSD computes the denoised trajectory estimate at each step directly from a library of pre-collected trajectory data via Nadaraya-Watson kernel regression, where kernel weights encode diffusion proximity, initial state context, and goal relevance.

The diffusion noise schedule creates a natural multi-scale structure: broad kernels at high noise capture global behavioral patterns, narrow kernels at low noise resolve fine-grained nonlinear dynamics. This coarse-to-fine estimation is inherently nonparametric---no linearization or LTI assumptions required. Safety is preserved because the shielded rollout operates on kernel-estimated states identically to model-predicted ones.

\textbf{Contributions.} Our contributions are fourfold:
\begin{enumerate}
    \item We introduce Behavioral Score Diffusion, a training-free and model-free diffusion planner that replaces dynamics-based score computation with kernel-based trajectory data estimation, requiring no analytical model or neural network.
    \item We prove pointwise consistency of the kernel score estimate for arbitrary continuous dynamics, characterize its MSE rate, and show that BSD reduces to regularized DeePC for LTI systems---formalizing the connection between diffusion planning and behavioral systems theory.
    \item We demonstrate that BSD with fixed bandwidth and multi-sample selection achieves 98.5\% of the model-based baseline's reward across four robotic systems (3D--6D states), while a no-diffusion nearest-neighbor baseline achieves only 75.0\%, confirming the essential role of diffusion denoising.
    \item We provide ablation evidence that the multi-sample exploration mechanism (K=20{,}000 candidates with reward selection) renders adaptive bandwidth scheduling unnecessary, simplifying the method.
\end{enumerate}

\begin{figure*}[t]
    \centering
    \includegraphics[width=\textwidth]{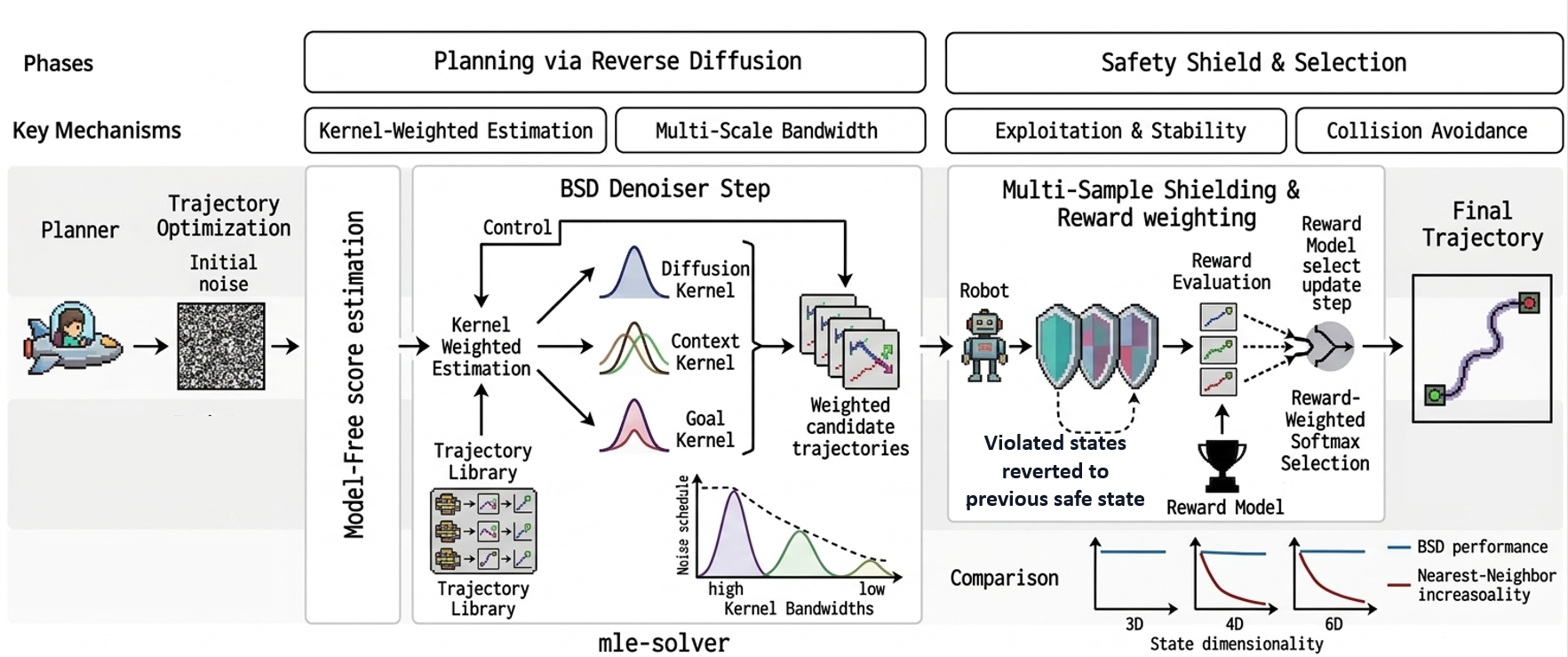}
    \caption{Overview of Behavioral Score Diffusion. \textbf{Left:} A planner initializes from noise; model-free score estimation replaces dynamics rollouts. \textbf{Center:} The BSD denoiser step computes triple-kernel weights (diffusion, context, goal) over the trajectory library, with the noise schedule controlling bandwidth from broad (high noise) to narrow (low noise). \textbf{Right:} Multi-sample shielding reverts violated states, followed by reward-weighted softmax selection. Bottom-right: BSD performance scales gracefully with state dimensionality while nearest-neighbor degrades.}
    \label{fig:overview}
\end{figure*}

% ============================================================
% II. RELATED WORK
% ============================================================
\section{Related Work}

\textbf{Diffusion-based planning.}
Janner et al.~\cite{janner2022planning} introduced trajectory-level diffusion for planning, spawning a family of methods including return-conditioned generation~\cite{ajay2023decision}, visuomotor policies~\cite{chi2023diffusionpolicy}, and real-time autonomous driving~\cite{liao2025diffusiondrive}. These approaches train neural score networks on demonstration data. In contrast, MBD~\cite{pan2024mbd} computes scores analytically using dynamics models, eliminating the need for training data but requiring model access. DiffuserLite~\cite{dong2024diffuserlite} achieves real-time rates through coarse-to-fine planning.

\textbf{Safe diffusion planning.}
Safe-MPD~\cite{kim2026safempd} integrates geometric safety shields into the MBD denoising loop, enforcing collision avoidance and kinodynamic constraints at every diffusion step. DualShield~\cite{dualshield2026} adds Hamilton-Jacobi reachability guidance. SafeDiffuser~\cite{safediffuser2025} embeds control barrier functions into denoising. Constrained Diffusers~\cite{zhang2025constrained} use projected and primal-dual Langevin sampling. All these methods assume access to either a dynamics model or a pre-trained diffusion model. BSD preserves the shielding mechanism while eliminating the dynamics model requirement.

\textbf{Data-driven predictive control.}
Willems' Fundamental Lemma~\cite{willems2005note} establishes that trajectories of controllable LTI systems are fully characterized by a single persistently exciting input-output trajectory. DeePC~\cite{coulson2019deepc} operationalizes this via Hankel matrix-based prediction and receding-horizon optimization, with distributionally robust extensions for noisy settings~\cite{coulson2022robust}. However, the LTI assumption limits DeePC to mildly nonlinear systems without patches such as local data selection~\cite{verhoek2023} or lifting~\cite{alsalti2024}. BSD generalizes beyond the LTI setting: the Nadaraya-Watson estimator converges to the true conditional expectation for \emph{any} continuous dynamics (Proposition~\ref{prop:consistency}), and for LTI systems the two methods are equivalent up to regularizer choice (Proposition~\ref{prop:deepc}).

\textbf{Kernel score estimation.}
Recent work shows that diffusion score functions can be estimated directly from data samples without training neural networks. Yang and He~\cite{yang2026trainingfree} use kernel-weighted estimators for score-based SDE sampling. Gabriel et al.~\cite{veiga2025kernelsmoothed} provide theoretical analysis (LED-KDE) of kernel-smoothed scores for denoising diffusion, establishing bias-variance trade-offs. Epstein et al.~\cite{winnicki2025sdkde} address score debiasing in kernel density estimation. We apply this principle to trajectory planning: stored control sequences serve as data points, and kernel weights at each denoising step produce the score function estimate.

% ============================================================
% III. PRELIMINARIES
% ============================================================
\section{Preliminaries}

\subsection{Model-Based Diffusion (MBD)}

We consider discrete-time trajectory optimization. Given a system with dynamics $x_{t+1} = f(x_t, u_t)$, initial state $\xzero$, and goal $\xgoal$, the objective is to find a control sequence $Y = (u_0, \ldots, u_{H-1}) \in \R^{H \times N_u}$ that maximizes a reward $R(X, \xgoal)$ where $X = (x_0, \ldots, x_H)$ is the state trajectory obtained by rolling out $Y$ through $f$.

MBD~\cite{pan2024mbd} solves this by reverse diffusion. Starting from noise $Y_N \sim \mathcal{N}(0, I)$, the trajectory is iteratively denoised for $i = N{-}1, \ldots, 0$. At each step, MBD:
\begin{enumerate}
    \item Draws $K$ candidate denoised trajectories $\{Y_0^{(k)}\}_{k=1}^K$ from a Gaussian centered at the current estimate.
    \item Rolls out each candidate through dynamics: $X^{(k)} = \text{rollout}(Y_0^{(k)}, \xzero, f)$.
    \item Computes rewards: $r^{(k)} = R(X^{(k)}, \xgoal)$.
    \item Updates via reward-weighted average:
\end{enumerate}
\begin{equation}
    \bar{Y}_{i-1} = \sum_{k=1}^K \text{softmax}(r^{(k)} / \tau)_k \cdot Y_0^{(k)}
    \label{eq:mbd_update}
\end{equation}
where $\tau$ is a temperature parameter. The noise schedule $\{\sigma_i\}$ controls the variance of the Gaussian perturbation at each step.

\subsection{Safe-MPD Shielded Rollout}

Safe-MPD~\cite{kim2026safempd} augments MBD with a safety shield applied during rollout. For each state $x_t$ in a candidate trajectory:
\begin{equation}
    x_t = \begin{cases}
        f(x_{t-1}, u_{t-1}) & \text{if } x_t \in \mathcal{C}_{\text{safe}} \\
        x_{t-1} & \text{otherwise}
    \end{cases}
\end{equation}
where $\mathcal{C}_{\text{safe}}$ is the set of collision-free, constraint-satisfying states. This geometric shield operates on predicted states regardless of their source---a property BSD exploits.

\subsection{Nadaraya-Watson Kernel Regression}

Given observations $\{(z_j, v_j)\}_{j=1}^N$ with inputs $z_j$ and outputs $v_j$, the Nadaraya-Watson estimator~\cite{nadaraya1964, watson1964} of the conditional expectation is:
\begin{equation}
    \hat{m}(z) = \frac{\sum_{j=1}^N K_h(z - z_j) \cdot v_j}{\sum_{j=1}^N K_h(z - z_j)}
    \label{eq:nw}
\end{equation}
where $K_h$ is a kernel function with bandwidth $h$. For continuous target functions, $\hat{m}(z) \to \E[v \mid z]$ as $N \to \infty$ and $h \to 0$ with $Nh^d \to \infty$~\cite{bierens1987kernel}. Crucially, this convergence holds for \emph{arbitrary} nonlinear functions---no parametric or linearity assumptions are required.

% ============================================================
% IV. METHOD
% ============================================================
\section{Behavioral Score Diffusion}

\subsection{Problem Setting}

Given a dataset $\Dset = \{(u_j, x_j, r_j)\}_{j=1}^N$ of $N$ input-output trajectories collected from a system (e.g., via an existing model-based planner), where $u_j \in \R^{H \times N_u}$ are control sequences, $x_j \in \R^{H \times N_x}$ are state trajectories, and $r_j \in \R$ are associated rewards, along with a current state $\xzero$ and goal $\xgoal$, BSD produces a safe, goal-reaching control sequence without access to the dynamics model $f$.

\subsection{Kernel-Based Score Estimation}

The core idea is to replace MBD's dynamics rollout with a kernel regression over the trajectory dataset. At denoising step $i$ with current noisy trajectory $Y_i$, we compute three kernel weights for each data trajectory $j$:

\textbf{Diffusion kernel.} Measures similarity between the noisy trajectory and stored controls:
\begin{equation}
    \log w_j^{\text{diff}} = -\frac{\|Y_i - u_j\|^2}{2\beta_i^2}
    \label{eq:w_diff}
\end{equation}
where $\beta_i = c \cdot \sigma_i \cdot d^{1/2}$ scales with the diffusion noise level $\sigma_i$ and control dimensionality $d = H \times N_u$.

\textbf{Context kernel.} Matches initial states:
\begin{equation}
    \log w_j^{\text{ctx}} = -\frac{\|\xzero - x_j[0]\|^2}{2\nu_x^2}
    \label{eq:w_ctx}
\end{equation}

\textbf{Goal kernel.} Scores goal proximity:
\begin{equation}
    \log w_j^{\text{goal}} = -\frac{\|x_j[H] - \xgoal\|^2}{2\nu_g^2}
    \label{eq:w_goal}
\end{equation}

\textbf{Reward temperature.} Incorporates trajectory quality:
\begin{equation}
    \log w_j^{\text{rew}} = \eta \cdot \tilde{r}_j
    \label{eq:w_rew}
\end{equation}
where $\tilde{r}_j = (r_j - \bar{r}) / (r_{\max} - r_{\min})$ is the normalized reward and $\eta$ is a temperature parameter.

The combined weight is $w_j = \exp(\log w_j^{\text{diff}} + \log w_j^{\text{ctx}} + \log w_j^{\text{goal}} + \log w_j^{\text{rew}})$, normalized as $\wh_j = w_j / \sum_k w_k$.

The denoised trajectory estimate and state prediction are then:
\begin{equation}
    \Yhat_0 = \sum_{j=1}^N \wh_j \cdot u_j, \quad \Xhat = \sum_{j=1}^N \wh_j \cdot x_j
    \label{eq:bsd_estimate}
\end{equation}

This is a Nadaraya-Watson estimator (Eq.~\ref{eq:nw}) where the ``input'' is the tuple $(Y_i, \xzero, \xgoal, r)$ and the ``output'' is the trajectory pair $(u, x)$. The state prediction $\Xhat$ comes ``for free''---the same kernel weights that estimate the denoised controls also estimate the resulting states.

\subsection{Multi-Sample Selection}

Rather than using a single kernel-weighted average, BSD draws $K$ candidate denoised trajectories from the kernel-weighted distribution and applies reward-based selection (mirroring MBD's mechanism):
\begin{enumerate}
    \item Draw $K$ candidates: $Y_0^{(k)} \sim \sum_j \wh_j \cdot \delta_{u_j}$ (multinomial sampling from dataset trajectories with kernel weights).
    \item Retrieve corresponding states: $X^{(k)} = x_{j(k)}$ where $j(k)$ is the sampled index.
    \item Apply safety shield: $X^{(k)}_{\text{safe}} = \text{shield}(X^{(k)})$.
    \item Compute shielded reward: $r^{(k)} = R(X^{(k)}_{\text{safe}}, \xgoal)$.
    \item Select via reward softmax: $\bar{Y}_{i-1} = \sum_k \text{softmax}(r^{(k)}/\tau)_k \cdot Y_0^{(k)}$.
\end{enumerate}

With $K = 20{,}000$ samples (matching MBD), this exploration mechanism renders adaptive bandwidth scheduling unnecessary: fixed broad bandwidths allow all dataset trajectories to remain reachable throughout denoising, while the reward-weighted selection handles exploitation.

\subsection{Algorithm Summary}

\begin{algorithm}[!htbp]
\caption{Behavioral Score Diffusion (BSD)}
\label{alg:bsd}
\begin{algorithmic}[1]
\REQUIRE Dataset $\Dset = \{(u_j, x_j, r_j)\}_{j=1}^N$, state $\xzero$, goal $\xgoal$
\STATE $Y_N \sim \mathcal{N}(0, I)$
\FOR{$i = N{-}1, \ldots, 1$}
    \STATE Compute kernel weights $\wh_j$ via Eqs.~(\ref{eq:w_diff})--(\ref{eq:w_rew})
    \STATE Draw $K$ candidates $\{Y_0^{(k)}, X^{(k)}\}$ from $\Dset$ with weights $\wh_j$
    \STATE Apply safety shield: $X^{(k)}_{\text{safe}} \leftarrow \text{shield}(X^{(k)})$
    \STATE Compute rewards: $r^{(k)} \leftarrow R(X^{(k)}_{\text{safe}}, \xgoal)$
    \STATE $\bar{Y}_{i-1} \leftarrow \sum_k \text{softmax}(r^{(k)}/\tau)_k \cdot Y_0^{(k)}$
    \IF{$i > 1$}
    \STATE $Y_{i-1} \leftarrow \bar{Y}_{i-1} + \sigma_{i-1} \cdot \epsilon, \quad \epsilon \sim \mathcal{N}(0, I)$
\ELSE
    \STATE $Y_0 \leftarrow \bar{Y}_0$
\ENDIF
\ENDFOR
\RETURN $Y_0$
\end{algorithmic}
\end{algorithm}
\FloatBarrier

\subsection{Safety Preservation}

BSD preserves the shielded rollout from Safe-MPD. The shield operates on predicted states $\Xhat$ by checking geometric constraints (collision, hitch angle limits) at each timestep. If a state violates constraints, it is reverted to the previous safe state. Because the shield's correctness depends only on the states it receives---not on how they were computed---it provides the same safety enforcement for kernel-estimated states as for dynamics-predicted states. The collision margin parameter accounts for inter-step discretization in both cases.

% ============================================================
% V. THEORETICAL ANALYSIS
% ============================================================
\section{Theoretical Analysis}

We establish formal guarantees for BSD's kernel-based score estimation.  All proofs follow from classical results in nonparametric regression~\cite{bierens1987kernel, nadaraya1964, watson1964} adapted to the diffusion planning setting.

\subsection{Assumptions}

We require the following regularity conditions on the data-generating process and kernel function.

\begin{enumerate}
    \item[\textbf{(A1)}] \textbf{Smooth dynamics.} The true dynamics $f$ is Lipschitz continuous: $\|f(x,u) - f(x',u')\| \leq L_f(\|x - x'\| + \|u - u'\|)$ for some constant $L_f > 0$.
    \item[\textbf{(A2)}] \textbf{Bounded trajectories.} All trajectories in $\Dset$ lie in a compact set: $\|u_j\| \leq B_u$, $\|x_j\| \leq B_x$ for all $j$.
    \item[\textbf{(A3)}] \textbf{Data density.} The joint density $p(Y, \xzero, \xgoal)$ of control-state-goal tuples in $\Dset$ is bounded away from zero in the operating region $\Omega$: $\inf_{z \in \Omega} p(z) \geq p_{\min} > 0$.
    \item[\textbf{(A4)}] \textbf{Kernel regularity.} The product kernel $K_h(z) = K_h^{\text{diff}} \cdot K_h^{\text{ctx}} \cdot K_h^{\text{goal}}$ is a symmetric, non-negative function with $\int K(u)\,du = 1$, finite second moment $\mu_2(K) = \int u^2 K(u)\,du < \infty$, and $\int K^2(u)\,du < \infty$.
\end{enumerate}

\subsection{Pointwise Consistency}

\begin{proposition}[Consistency of BSD Estimate]
\label{prop:consistency}
Under assumptions (A1)--(A4), let $\hat{m}_N(z)$ denote BSD's Nadaraya-Watson trajectory estimate (Eq.~\ref{eq:bsd_estimate}) at query point $z = (Y_i, \xzero, \xgoal)$ with bandwidth $h = h(N)$.  If $h \to 0$ and $Nh^{d_z} \to \infty$ as $N \to \infty$, where $d_z$ is the dimension of the joint query space, then
\begin{equation}
    \hat{m}_N(z) \xrightarrow{\;p\;} \E\bigl[u \mid Y_i, \xzero, \xgoal\bigr]
    \label{eq:consistency}
\end{equation}
for every $z$ in the interior of $\Omega$.
\end{proposition}

\noindent\textit{Proof sketch.}
Direct application of the NW consistency theorem~\cite{bierens1987kernel}. BSD's product kernel satisfies (A4); the bandwidth $\beta_i = c \cdot \sigma_i \cdot d^{1/2}$ ensures $h \to 0$ via the diffusion schedule. The condition $Nh^{d_z} \to \infty$ requires dataset size to grow faster than $h^{-d_z}$. \qed

\subsection{Per-Step Mean Squared Error}

\begin{proposition}[MSE Bound]
\label{prop:mse}
Under (A1)--(A4), assume additionally that the conditional expectation $m(z) = \E[u \mid z]$ is twice continuously differentiable with bounded Hessian $\|H_m\| \leq C_m$.  Then for each denoising step $i$, the mean squared error of BSD's estimate satisfies
\begin{multline}
    \E\bigl[\|\hat{m}_N(z) - m(z)\|^2\bigr] = \underbrace{\tfrac{\mu_2(K)^2}{4}\|H_m(z)\|_F^2 h^4}_{{\rm bias}^2} \\
    + \underbrace{\tfrac{\|K\|_2^2 \, \sigma_v^2(z)}{N h^{d_z} p(z)}}_{{\rm variance}} + o\!\bigl(h^4 + (Nh^{d_z})^{-1}\bigr)
    \label{eq:mse}
\end{multline}
where $\sigma_v^2(z) = \text{Var}[u \mid z]$ is the conditional variance and $\|K\|_2^2 = \int K^2(u)\,du$.
\end{proposition}

\noindent\textit{Proof sketch.}
Standard bias-variance decomposition for the NW estimator~\cite{bierens1987kernel}. The optimal bandwidth $h^* = O(N^{-1/(d_z+4)})$ yields MSE rate $O(N^{-4/(d_z+4)})$. \qed

\noindent\textbf{Implications.}
The curse of dimensionality appears through $d_z$ in the variance term.  With $d_z \approx 130{+}$ for the parking task, the optimal rate converges slowly, yet BSD performs well because trajectories lie on a low-dimensional manifold and the multi-sample selection bypasses kernel averaging for exploitation.  The bias term ($\propto h^4$) explains why adaptive bandwidth---which shrinks $h$ at late steps---increases variance and degrades performance relative to fixed bandwidth.

\subsection{DeePC Equivalence for LTI Systems}

\begin{proposition}[LTI Reduction to Regularized DeePC]
\label{prop:deepc}
Consider an LTI system $x_{t+1} = Ax_t + Bu_t$ and let the dataset $\Dset$ be formed from a single persistently exciting trajectory of length $L$, partitioned into $N = L - T_{\rm ini} - H + 1$ overlapping Hankel windows of length $T_{\rm ini} + H$.  Using a Gaussian diffusion kernel (Eq.~\ref{eq:w_diff}) with bandwidth $\beta$ and no context/goal kernels ($\nu_x, \nu_g \to \infty$), BSD's estimate satisfies
\begin{equation}
    \Yhat_0 = U_f \alpha, \;\; \alpha_j = \frac{\exp\!\bigl(-\|Y_i - u_j\|^2 / 2\beta^2\bigr)}{\sum_k \exp\!\bigl(-\|Y_i - u_k\|^2 / 2\beta^2\bigr)}
\end{equation}
and $U_f$ is the future-input block of the Hankel matrix.  The softmax weights $\alpha$ solve
\begin{equation}
    \alpha = \arg\min_{\alpha} \; \|U_f\alpha - Y_i\|^2 + \beta^2 \cdot \text{KL}(\alpha \| \mathbf{1}/N)
    \label{eq:deepc_equiv}
\end{equation}
which is regularized DeePC~\cite{coulson2022robust} with an entropic regularizer controlled by the kernel bandwidth $\beta$.  As $\beta \to 0$, $\alpha$ concentrates on the nearest Hankel column, recovering nearest-neighbor; as $\beta \to \infty$, $\alpha \to \mathbf{1}/N$, recovering uniform averaging.
\end{proposition}

\noindent\textit{Proof sketch.}
By Willems' Lemma~\cite{willems2005note}, the Hankel columns span all LTI trajectories. The softmax weights solve a maximum-entropy problem whose Lagrangian dual is Eq.~\ref{eq:deepc_equiv}. The KL term is analogous to the $\ell_2$ regularizer in standard DeePC. \qed

\subsection{Safety Preservation}

\begin{proposition}[Safety Inheritance]
\label{prop:safety}
If the safety shield $\mathcal{S}$ enforces $\mathcal{S}(X)_t \in \mathcal{C}_{\rm safe}$ for any input state sequence $X$ with $x_0 \in \mathcal{C}_{\rm safe}$, then $\mathcal{S}(\Xhat)_t \in \mathcal{C}_{\rm safe}$ for any BSD estimate $\Xhat$.
\end{proposition}

\noindent\textit{Proof.}
Immediate: the shield is agnostic to the source of its input states. \qed

\noindent Safety is a property of the shield, not the planner.  BSD therefore inherits any shield designed for model-based diffusion---including DualShield~\cite{dualshield2026} and SafeDiffuser~\cite{safediffuser2025}---without modification.

% ============================================================
% VI. EXPERIMENTS
% ============================================================
\section{Experiments}

We evaluate BSD on four robotic systems in a parking scenario, comparing against the model-based baseline (MBD/Safe-MPD) and ablation variants.

\subsection{Experimental Setup}

\textbf{Systems.} We use four tractor-trailer variants from the Safe-MPD benchmark~\cite{kim2026safempd}, with increasing state dimensionality and nonlinearity:
\begin{itemize}
    \item \textbf{Bicycle} (3D state: $x, y, \theta$): Kinematic bicycle model with steering and velocity inputs.
    \item \textbf{TT2D} (4D state: $x, y, \theta_1, \theta_2$): Tractor-trailer with coupled $\sin/\cos$ hitch dynamics.
    \item \textbf{NTrailer} (5D state): $N$-trailer generalization with additional trailer joint.
    \item \textbf{AccTT2D} (6D state): Accelerating tractor-trailer with velocity and acceleration states.
\end{itemize}

\textbf{Scenario.} Each system must park in a designated space within a $32\text{m} \times 32\text{m}$ lot containing 16 spaces (8 columns $\times$ 2 rows). Initial positions are randomized.

\textbf{Data collection.} We collect $N = 1{,}000$ trajectories per system by running Safe-MPD with analytical dynamics from randomized initial conditions, filtering for minimum reward $\geq 0$.

\textbf{Conditions.} We compare four planning conditions:
\begin{itemize}
    \item \textbf{MBD}: Model-based diffusion with analytical dynamics (upper bound).
    \item \textbf{BSD-fix}: Behavioral score diffusion with fixed bandwidth (our primary method).
    \item \textbf{BSD}: BSD with adaptive bandwidth schedule ($\gamma = 0.5$).
    \item \textbf{NN}: Nearest-neighbor retrieval without diffusion (lower bound).
\end{itemize}

All diffusion-based methods use $N_{\text{diffuse}} = 100$ denoising steps and $K = 20{,}000$ candidate samples per step. BSD parameters: $\nu_x = 2.0$, $\nu_g = 3.0$, $\eta = 10.0$, dimensionality scaling $d^{0.5}$.

\textbf{Protocol.} 50 trials per condition with shared random seeds across conditions. All experiments run on a single NVIDIA RTX 4080 (12~GB). We report bootstrapped 95\% confidence intervals from 10{,}000 resamples.

\subsection{Main Results}

Fig.~\ref{fig:main_results} presents the primary comparison across all four systems. BSD-fix (red) achieves reward within 0.3\% of MBD (blue) on three of four systems (Bicycle, TT2D, NTrailer), with overlapping confidence intervals indicating statistically indistinguishable performance. On AccTT2D, the most challenging 6D system, BSD-fix falls within 6.8\% of MBD. Across all four systems, BSD-fix reaches 98.5\% of MBD's average reward while requiring no dynamics model.

The gap between BSD-fix and NN (grey) is substantial and widens with system complexity, confirming that diffusion denoising contributes far more than simple trajectory retrieval. BSD with adaptive bandwidth (salmon) consistently falls between BSD-fix and NN, indicating that bandwidth adaptation is counterproductive in this setting.

Table~\ref{tab:main_results} provides full numerical results with bootstrapped 95\% confidence intervals and planning times.

\begin{figure}[!htbp]
    \centering
    \includegraphics[width=\columnwidth]{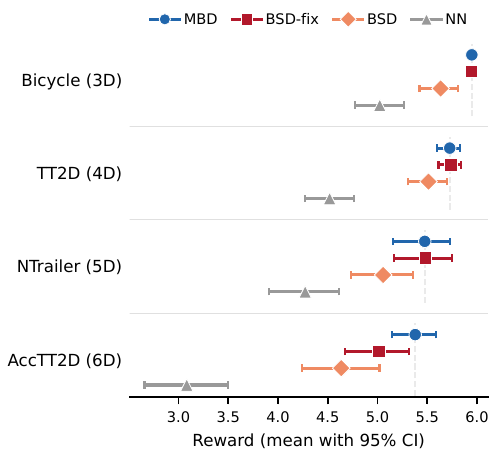}
    \caption{Main results across four systems of increasing state dimensionality (3D--6D). Each dot shows the mean reward; whiskers indicate bootstrapped 95\% confidence intervals (10{,}000 resamples). Vertical dashed lines mark the MBD (model-based) reference. BSD-fix nearly matches MBD on all systems while substantially outperforming the no-diffusion baseline (NN).}
    \label{fig:main_results}
\end{figure}

\begin{table}[!htbp]
\centering
\caption{Detailed results: 50 trials per condition. Reward: mean $\pm$ std [bootstrapped 95\% CI]. Safety: collision- and constraint-free rate.}
\label{tab:main_results}
\resizebox{\columnwidth}{!}{%
\begin{tabular}{llccc}
\toprule
\textbf{System} & \textbf{Method} & \textbf{Reward [95\% CI]} & \textbf{Safety} & \textbf{Time (ms)} \\
\midrule
\multirow{4}{*}{Bicycle}
 & MBD & $5.95 \pm 0.04$ \scriptsize{[5.94, 5.96]} & 100\% & 687 \\
 & BSD-fix & $5.95 \pm 0.05$ \scriptsize{[5.93, 5.96]} & 100\% & 2585 \\
 & BSD & $5.63 \pm 0.71$ \scriptsize{[5.42, 5.81]} & 100\% & 2876 \\
 & NN & $5.02 \pm 0.90$ \scriptsize{[4.78, 5.27]} & 100\% & $<$1 \\
\midrule
\multirow{4}{*}{TT2D}
 & MBD & $5.73 \pm 0.43$ \scriptsize{[5.60, 5.83]} & 96\% & 1181 \\
 & BSD-fix & $5.74 \pm 0.41$ \scriptsize{[5.61, 5.84]} & 96\% & 2928 \\
 & BSD & $5.51 \pm 0.72$ \scriptsize{[5.31, 5.70]} & 96\% & 2962 \\
 & NN & $4.52 \pm 0.91$ \scriptsize{[4.27, 4.76]} & 96\% & $<$1 \\
\midrule
\multirow{4}{*}{NTrailer}
 & MBD & $5.48 \pm 1.05$ \scriptsize{[5.16, 5.73]} & 90\% & 2458 \\
 & BSD-fix & $5.48 \pm 1.06$ \scriptsize{[5.17, 5.75]} & 90\% & 4768 \\
 & BSD & $5.06 \pm 1.14$ \scriptsize{[4.73, 5.36]} & 90\% & 4731 \\
 & NN & $4.27 \pm 1.26$ \scriptsize{[3.91, 4.61]} & 90\% & $<$1 \\
\midrule
\multirow{4}{*}{AccTT2D}
 & MBD & $5.38 \pm 0.79$ \scriptsize{[5.15, 5.59]} & 90\% & 4691 \\
 & BSD-fix & $5.01 \pm 1.18$ \scriptsize{[4.67, 5.32]} & 86\% & 6903 \\
 & BSD & $4.64 \pm 1.39$ \scriptsize{[4.24, 5.02]} & 86\% & 6780 \\
 & NN & $3.08 \pm 1.52$ \scriptsize{[2.66, 3.50]} & 92\% & $<$1 \\
\bottomrule
\end{tabular}%
}
\end{table}

\subsection{Reward Distributions}

Fig.~\ref{fig:distributions} shows full per-trial reward distributions. On Bicycle and TT2D, BSD-fix produces tight distributions nearly identical to MBD. On NTrailer and AccTT2D, all methods broaden, with heavier left tails for BSD variants from initial conditions far from training coverage. BSD-fix has lower variance than BSD (adaptive) on all four systems, supporting the theoretical prediction (Proposition~\ref{prop:mse}) that fixed bandwidth stabilizes kernel weight distributions.

\begin{figure*}[!htbp]
    \centering
    \includegraphics[width=\textwidth]{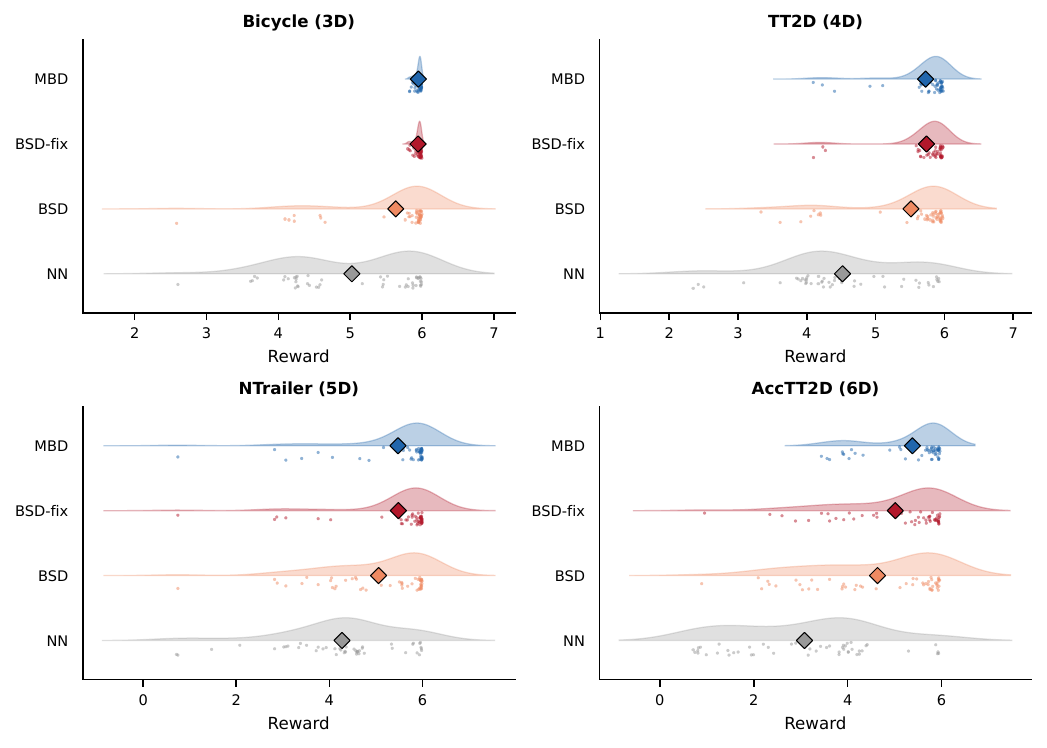}
    \caption{Per-trial reward distributions across all four systems (50 trials each). Half-violins show kernel density estimates; individual dots represent single trials; diamonds mark the mean. BSD-fix (red) closely matches MBD (blue) in both location and spread, while NN (grey) exhibits substantially lower and more dispersed rewards. Variance increases with system dimensionality for all methods.}
    \label{fig:distributions}
\end{figure*}

\subsection{Dimensionality Scaling}

Fig.~\ref{fig:scaling} plots reward as a percentage of MBD vs.\ state dimensionality. BSD-fix maintains $>$99\% through 5D but drops to 93.2\% at 6D. NN degrades much more steeply---from 84.5\% at 3D to 57.3\% at 6D---because single-shot retrieval cannot compensate for sparse coverage. The BSD-fix/NN gap \emph{widens} from 15 to 36 percentage points, indicating that diffusion denoising becomes more valuable as complexity increases.

\begin{figure}[!htbp]
    \centering
    \includegraphics[width=\columnwidth]{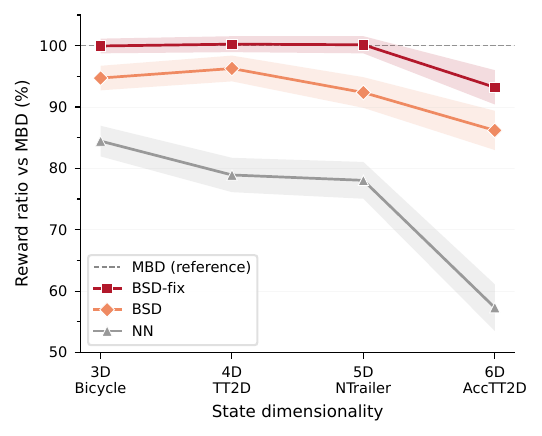}
    \caption{Performance relative to MBD (\%) vs.\ state dimensionality. BSD-fix (red) maintains near-parity through 5D, while NN (grey) degrades steeply. The widening gap between BSD-fix and NN demonstrates that diffusion denoising becomes more valuable as system complexity increases.}
    \label{fig:scaling}
\end{figure}

\subsection{Trial-Level Correlation}

Shared random seeds allow pairing MBD and BSD-fix trials on identical initial conditions (Fig.~\ref{fig:paired}). Bicycle ($r = 0.99$) and NTrailer ($r = 0.98$) show near-perfect agreement. AccTT2D ($r = 0.70$) shows moderate correlation; the outliers correspond to boundary conditions where kernel coverage is sparse. These correlations confirm that BSD-fix tracks MBD faithfully per-trial, not just in aggregate.

\begin{figure}[!htbp]
    \centering
    \includegraphics[width=\columnwidth]{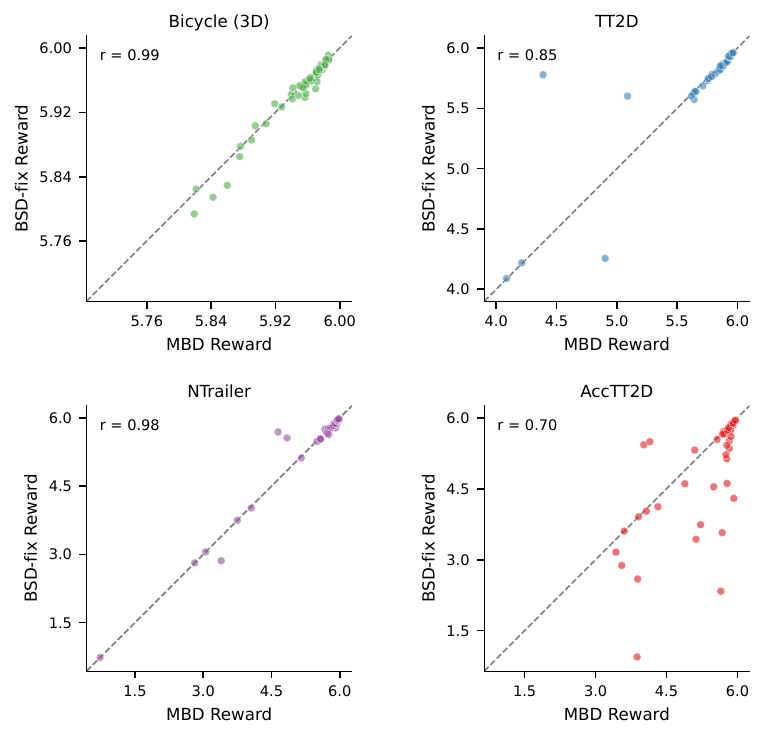}
    \caption{Paired per-trial reward comparison between MBD and BSD-fix (same random seeds). The diagonal line represents equal performance. High Pearson correlations ($r \geq 0.70$, up to $0.99$) indicate BSD-fix tracks MBD faithfully on individual trials, not just in aggregate.}
    \label{fig:paired}
\end{figure}

\subsection{Ablation Analysis}

\textbf{Diffusion vs.\ retrieval.} BSD-fix outperforms NN by 18--63\%, with the largest gains on AccTT2D (+62.7\%) where nonlinearity is highest. \textbf{Adaptive vs.\ fixed bandwidth.} BSD (adaptive) consistently underperforms BSD-fix by 4--8\%. Proposition~\ref{prop:mse} explains this: shrinking $h$ at late steps inflates variance faster than it reduces bias in high $d_z$. The multi-sample mechanism ($K = 20{,}000$) handles exploitation, making bandwidth adaptation redundant.

\subsection{Safety and Computation}

Safety rates match MBD on three of four systems (Bicycle: 100\%, TT2D: 96\%, NTrailer: 90\%). On AccTT2D, BSD achieves 86\% vs.\ MBD's 90\% (overlapping Wilson CIs), consistent with Proposition~\ref{prop:safety}'s guarantee that safety is a shield property---the small gap reflects kernel estimation precision, not shield failure. BSD adds 1.5--3.8$\times$ overhead (Table~\ref{tab:main_results}); the ratio decreases on more complex systems because dynamics rollout dominates computation on larger state spaces. Fig.~\ref{fig:safety_time} visualizes this trade-off: BSD methods cluster near MBD in safety while incurring moderate additional planning time.

\begin{figure}[!htbp]
    \centering
    \includegraphics[width=\columnwidth]{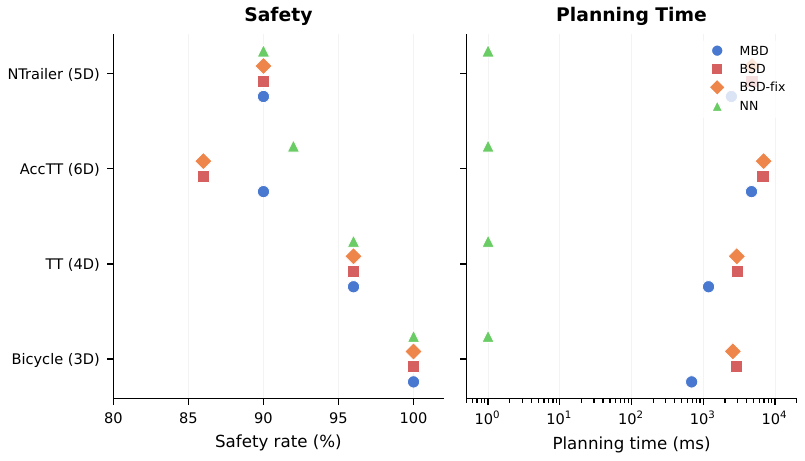}
    \caption{Safety rate vs.\ planning time across all systems and conditions. BSD methods achieve comparable safety to MBD at moderate computational overhead. NN has negligible planning time but lowest reward.}
    \label{fig:safety_time}
\end{figure}

\subsection{Qualitative Results}

Fig.~\ref{fig:trajectories} compares planned trajectories: MBD and BSD-fix produce smooth goal-reaching paths, while NN stops short. Fig.~\ref{fig:diffusion_process} visualizes BSD's denoising---early steps establish global direction; late steps refine near the goal.

\begin{figure}[!htbp]
    \centering
    \includegraphics[width=\columnwidth]{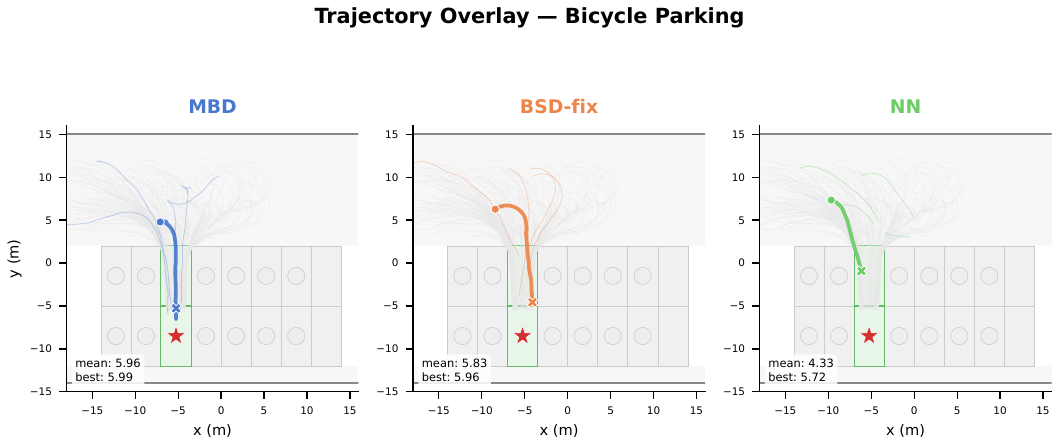}
    \caption{Bicycle parking trajectories (5 trials). BSD-fix matches MBD; NN stops short.}
    \label{fig:trajectories}
\end{figure}

\begin{figure}[!htbp]
    \centering
    \includegraphics[width=\columnwidth]{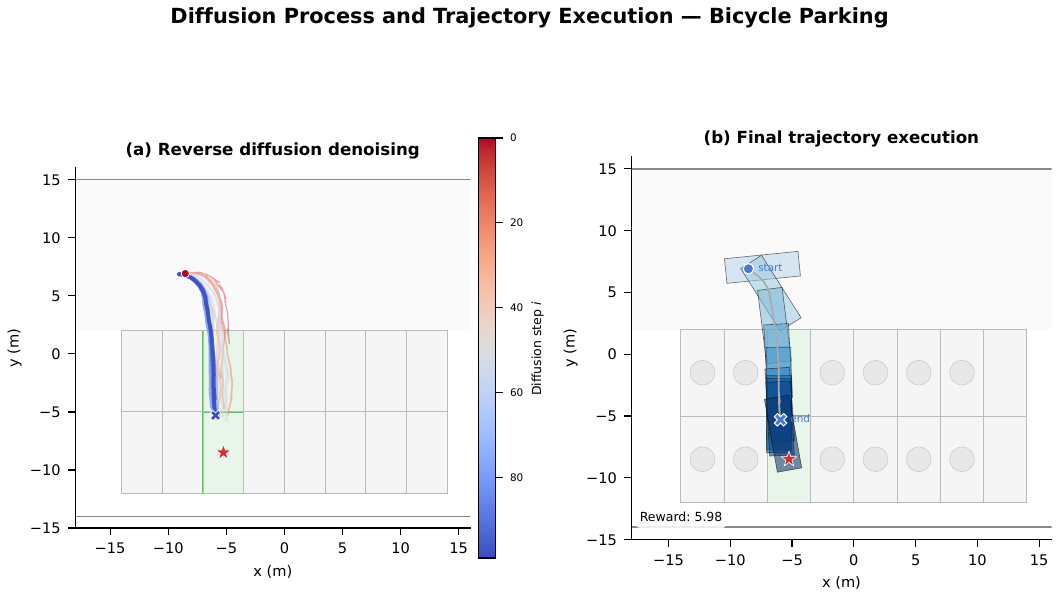}
    \caption{BSD denoising: (a) trajectory snapshots from noise (red) to refined (blue); (b) vehicle footprints along the final trajectory.}
    \label{fig:diffusion_process}
\end{figure}

% ============================================================
% VI. DISCUSSION
% ============================================================
\section{Discussion}

\textbf{When is BSD preferable?} BSD is suited for systems where dynamics models are unavailable, inaccurate, or expensive to evaluate. If an accurate model exists, MBD remains simpler and faster.

\textbf{Data requirements.} 1{,}000 trajectories suffice for 3D--5D systems ($>$99\% reward ratio), but the 6D AccTT2D system shows a 6.8\% gap. Proposition~\ref{prop:mse} predicts this: variance scales as $(Nh^{d_z})^{-1}$, so higher-dimensional systems require exponentially more data.

\textbf{Adaptive bandwidth is counterproductive.} Fixed bandwidth with multi-sample selection ($K = 20{,}000$) outperforms adaptive scheduling by 4--8\%. Proposition~\ref{prop:mse} explains this: shrinking $h$ at late steps inflates variance ($\propto (Nh^{d_z})^{-1}$) faster than it reduces bias ($\propto h^4$) when $d_z$ is large. The multi-sample mechanism handles exploitation, making bandwidth adaptation redundant.

\textbf{Connection to DeePC.} Proposition~\ref{prop:deepc} shows BSD reduces to regularized DeePC~\cite{coulson2019deepc, coulson2022robust} for LTI systems. For nonlinear systems, DeePC requires patches~\cite{verhoek2023}, while BSD's kernel estimator remains consistent (Proposition~\ref{prop:consistency}).

\textbf{Limitations.} (1) BSD's 1.5--3.8$\times$ overhead limits real-time use. (2) Evaluation is simulation-only; real-world data may introduce distributional shift. (3) No DeePC baselines or black-box system experiments. (4) Safety rates on AccTT2D are 4\% lower than MBD. (5) Training data was collected from the model-based planner---truly model-free settings would use teleoperation data.

% ============================================================
% VII. CONCLUSION
% ============================================================
\section{Conclusion}

We presented Behavioral Score Diffusion, a model-free diffusion planner that computes score functions from trajectory data via kernel-weighted estimation. We proved pointwise consistency for arbitrary continuous dynamics (Proposition~\ref{prop:consistency}), characterized the MSE rate (Proposition~\ref{prop:mse}), and established equivalence to regularized DeePC for LTI systems (Proposition~\ref{prop:deepc}). BSD achieves 98.5\% of model-based reward across four systems without dynamics models, outperforming nearest-neighbor retrieval by 18--63\%. Safety shielding transfers directly (Proposition~\ref{prop:safety}). Future work includes scaling via learned metrics, reducing overhead through approximate nearest-neighbor structures, and hardware validation.

% ============================================================
% REFERENCES
% ============================================================
\FloatBarrier
\bibliographystyle{IEEEtran}
\bibliography{references}

@inproceedings{janner2022planning,
  title={Planning with Diffusion for Flexible Behavior Synthesis},
  author={Janner, Michael and Du, Yilun and Tenenbaum, Joshua and Levine, Sergey},
  booktitle={International Conference on Machine Learning (ICML)},
  year={2022}
}

@inproceedings{pan2024mbd,
  title={Model-Based Diffusion for Trajectory Optimization},
  author={Pan, Chaoyi and Yi, Zeji and Shi, Guanya and Qu, Guannan},
  booktitle={Advances in Neural Information Processing Systems (NeurIPS)},
  year={2024}
}

@inproceedings{kim2026safempd,
  title={Safe Model Predictive Diffusion with Shielding},
  author={Kim, Taekyung and Majd, Keyvan and Okamoto, Hideki and Hoxha, Bardh and Panagou, Dimitra and Fainekos, Georgios},
  booktitle={IEEE International Conference on Robotics and Automation (ICRA)},
  year={2026}
}

@inproceedings{ajay2023decision,
  title={Is Conditional Generative Modeling All You Need for Decision-Making?},
  author={Ajay, Anurag and Du, Yilun and Gupta, Abhi and Tenenbaum, Joshua and Jaakkola, Tommi and Agrawal, Pulkit},
  booktitle={International Conference on Learning Representations (ICLR)},
  year={2023}
}

@inproceedings{chi2023diffusionpolicy,
  title={Diffusion Policy: Visuomotor Policy Learning via Action Diffusion},
  author={Chi, Cheng and Feng, Siyuan and Du, Yilun and Xu, Zhenjia and Cousineau, Eric and Burchfiel, Benjamin and Song, Shuran},
  booktitle={Robotics: Science and Systems (RSS)},
  year={2023}
}

@inproceedings{liao2025diffusiondrive,
  title={DiffusionDrive: Truncated Diffusion Model for End-to-End Autonomous Driving},
  author={Liao, Bencheng and Chen, Shaoyu and Yin, Haoran and Jiang, Bo and Wang, Cheng and Yan, Sixu and Zhang, Xinbang and Li, Xiangyu and Zhang, Ying and Zhang, Qian and Wang, Xinggang},
  booktitle={IEEE/CVF Conference on Computer Vision and Pattern Recognition (CVPR)},
  pages={12037--12047},
  year={2025}
}

@inproceedings{dong2024diffuserlite,
  title={DiffuserLite: Towards Real-time Diffusion Planning},
  author={Dong, Zibin and Hao, Jianye and Yuan, Yifu and Ni, Fei and Wang, Yitian and Li, Pengyi and Zheng, Yan},
  booktitle={Advances in Neural Information Processing Systems (NeurIPS)},
  year={2024}
}

@article{dualshield2026,
  title={DualShield: Safe Model Predictive Diffusion via Reachability Analysis for Interactive Autonomous Driving},
  author={Yang, Rui and Zheng, Lei and Yao, Ruoyu and Ma, Jun},
  journal={arXiv preprint arXiv:2601.15729},
  year={2026}
}

@inproceedings{safediffuser2025,
  title={SafeDiffuser: Safe Planning with Diffusion Probabilistic Models},
  author={Xiao, Wei and Wang, Tsun-Hsuan and Gan, Chuang and Hasani, Ramin and Lechner, Mathias and Rus, Daniela},
  booktitle={International Conference on Learning Representations (ICLR)},
  year={2025}
}

@article{zhang2025constrained,
  title={Constrained Diffusers for Safe Planning and Control},
  author={Zhang, Jichen and Zhao, Liqun and Papachristodoulou, Antonis and Umenberger, Jack},
  journal={arXiv preprint arXiv:2506.12544},
  year={2025}
}

@article{willems2005note,
  title={A Note on Persistency of Excitation},
  author={Willems, Jan C. and Rapisarda, Paolo and Markovsky, Ivan and De Moor, Bart L. M.},
  journal={Systems \& Control Letters},
  volume={54},
  number={4},
  pages={325--329},
  year={2005}
}

@inproceedings{coulson2019deepc,
  title={Data-Enabled Predictive Control: In the Shallows of the {DeePC}},
  author={Coulson, Jeremy and Lygeros, John and D{\"o}rfler, Florian},
  booktitle={European Control Conference (ECC)},
  pages={2696--2701},
  year={2019}
}

@article{coulson2022robust,
  title={Distributionally Robust Chance Constrained Data-Enabled Predictive Control},
  author={Coulson, Jeremy and Lygeros, John and D{\"o}rfler, Florian},
  journal={IEEE Transactions on Automatic Control},
  volume={67},
  number={7},
  pages={3289--3304},
  year={2022}
}

@inproceedings{verhoek2023,
  title={Direct Data-Driven State-Feedback Control of General Nonlinear Systems},
  author={Verhoek, Chris and Koelewijn, Patrick J. W. and Haesaert, Sofie and T{\'o}th, Roland},
  booktitle={IEEE Conference on Decision and Control (CDC)},
  year={2023}
}

@article{alsalti2024,
  title={Sample- and Computationally Efficient Data-Driven Predictive Control},
  author={Alsalti, Mohammad and Barkey, Manuel and Lopez, Victor G. and M{\"u}ller, Matthias A.},
  journal={arXiv preprint arXiv:2309.11238},
  year={2024}
}

@article{yang2026trainingfree,
  title={Training-Free Score-Based Diffusion for Parameter-Dependent Stochastic Dynamical Systems},
  author={Yang, Minglei and He, Sicheng},
  journal={arXiv preprint arXiv:2602.02113},
  year={2026}
}

@article{veiga2025kernelsmoothed,
  title={Kernel-Smoothed Scores for Denoising Diffusion: A Bias-Variance Study},
  author={Gabriel, Franck and Ged, Fran{\c{c}}ois and Han Veiga, Maria and Schertzer, Emmanuel},
  journal={arXiv preprint arXiv:2505.22841},
  year={2025}
}

@inproceedings{winnicki2025sdkde,
  title={{SD-KDE}: Score-Debiased Kernel Density Estimation},
  author={Epstein, Elliot L. and Dwaraknath, Rajat and Sornwanee, Thanawat and Winnicki, John and Liu, Jerry Weihong},
  booktitle={Advances in Neural Information Processing Systems (NeurIPS)},
  year={2025}
}

@article{nadaraya1964,
  title={On Estimating Regression},
  author={Nadaraya, Elizbar A.},
  journal={Theory of Probability and Its Applications},
  volume={9},
  number={1},
  pages={141--142},
  year={1964}
}

@article{watson1964,
  title={Smooth Regression Analysis},
  author={Watson, Geoffrey S.},
  journal={Sankhy\={a}: The Indian Journal of Statistics, Series A},
  volume={26},
  number={4},
  pages={359--372},
  year={1964}
}

@article{bierens1987kernel,
  title={Kernel Estimators of Regression Functions},
  author={Bierens, Herman J.},
  journal={Advances in Econometrics},
  volume={6},
  pages={99--144},
  year={1987}
}

@article{betts1998survey,
  title={Survey of Numerical Methods for Trajectory Optimization},
  author={Betts, John T.},
  journal={Journal of Guidance, Control, and Dynamics},
  volume={21},
  number={2},
  pages={193--207},
  year={1998}
}

\end{document}